\documentclass[final]{siamart171218}

\usepackage{amsmath,amssymb}
\usepackage{nicefrac}
\usepackage{defs} % madeleine's favorite definitions
\usepackage{enumitem}

\usepackage{todonotes}

\renewcommand{\reals}{\mathbb{R}}
\newcommand{\ubeta}{\smash{\beta}}

\newcommand{\ualpha}{\smash{\alpha}}

\newtheorem*{remark}{Remark}

\title{Why are Big Data Matrices Approximately Low Rank?}

\author{Madeleine Udell\thanks{Department of Operations Research and Information Engineering, Cornell University, Ithaca, NY 14853. (\texttt{udell@cornell.edu}) This work is supported by DARPA Award No.~FA8750-17-2-0101.} \and
Alex Townsend\thanks{Department of Mathematics, Cornell University, Ithaca, NY 14853. (\texttt{townsend@cornell.edu}) This work is supported by National Science Foundation grant No.~1645445.}}

\date{\today}

\begin{document}

\maketitle

\begin{abstract}
Matrices of (approximate) low rank are pervasive in data science, appearing in recommender systems, movie preferences,
topic models, medical records, and genomics.
While there is a vast literature on how to exploit low rank structure in these datasets,
there is less attention on explaining why the low rank structure appears in the first place.
Here, we explain the effectiveness of low rank models in data science
by considering a simple generative model for these matrices:
we suppose that each row or column is associated to a (possibly high dimensional)
bounded latent variable,
and entries of the matrix are generated by applying a piecewise analytic function
to these latent variables.
These matrices are in general full rank.
However, we show that we can approximate every entry of
an $m \times n$ matrix drawn from this model
to within a fixed absolute error
by a low rank matrix whose rank grows as $\mathcal O(\log(m + n))$.
% We show that for any $\epsilon$, we can approximate every entry of
% an $m \times n$ matrix drawn from such a model to within $\epsilon$ absolute error
% by a matrix whose rank grows as $r = \mathcal O(\log(m + n)/\epsilon^2)$.
Hence any sufficiently large matrix from such a latent variable model can be approximated,
up to a small entrywise error, by a low rank matrix.

\end{abstract}

\section{Introduction}
Low rank matrices appear throughout the sciences in computational mathematics~\cite{Bebendorf_08_01}, statistics~\cite{Giraud_14_01}, and machine learning~\cite{Kulis_06_01}. Numerous techniques have been developed over the last 50 years to exploit low rank structure whenever it appears, whether in movie preferences~\cite{funk2006,bell2007lessons}, social networks~\cite{liben2007link,menon2011link}, genomics~\cite{brunet2004metagenes,gao2005improving,kim2007sparse,witten2009penalized}, medical records~\cite{schuler2016discovering}, or text documents~\cite{deerwester1990indexing,dhillon2001co,pennington2014}.

It is useful to know when a dataset can be approximated by a low rank matrix. A low rank approximation can be used to make filtering and statistics either computationally feasible or more efficient. In machine learning, low rank approximations to data tables are often employed to impute missing data, denoise noisy data, or perform feature extraction~\cite{Udell_16_01}. These techniques are also fundamental for many algorithms in recommender systems~\cite{koren2009matrix}.

The broad applicability of low rank techniques is at first rather puzzling. Since the set of singular matrices is nowhere dense, random (``average'') matrices are almost surely of full rank. In addition, the singular values of random Gaussian matrices are large with extraordinarily high probability~\cite{Edelman_88_01}. We must conclude that matrices and datasets that appear in the real-world must be far from average. We would like to understand the underlying phenomena that generate compressible datasets.

Let us begin with a statement about any matrix.
\addtocounter{theorem}{-1}
\begin{theorem}
Let $X\in\mathbb{R}^{n\times n}$ and $0<\epsilon<1$. Then, with $r = \lceil 72\log(2n+1)/\epsilon^2\rceil$ we have
\begin{equation}
\inf_{{\rm rank}(Y)\leq r} \|X-Y\|_{\max} \leq \epsilon \|X\|_2,
\label{eq:SimpleBound}
\end{equation}
where $\|\cdot\|_{\max}$ is the maximum absolute entry norm and $\|\cdot\|_2$ is the spectral matrix norm.
\label{thm:SimpleTheorem}
\end{theorem}
% Theorem~\ref{thm-simple} states that any matrix $X\in\mathbb{R}^{n\times n}$
% can be approximated entrywise by a matrix $Y \in\mathbb{R}^{n\times n}$ whose rank
% grows slowly with the dimension $n$,
% with a quality of approximation depending on $\|X\|_2$.
To avoid the theorem collapsing to a trivial statement,
we need $r<n$, which only occurs when $n$ is extremely large.
Hence roughly speaking,
the theorem says that any sufficiently large matrix with a small spectral norm
can be well approximated entrywise by a low rank matrix.

It is important to appreciate that Theorem~\ref{thm:SimpleTheorem} above holds for {\em any} matrix, including the identity matrix, and that the result is trivially false if $\|\cdot\|_{\max}$ is replaced by $\|\cdot\|_2$.
Spectral norm approximations are generally preferred in linear algebra,
however, for data science applications an entrywise approximation is generally much more important.
In a data science setting, one often wants to compress a dataset while perturbing each entry as little as possible --- this is exactly what the maximum absolute entry norm captures.

% Why are low rank matrices so prevalent?
% Is there is an underlying mathematical phenomenon that explains the
% abundance of low rank structures in data science?
Theorem~\ref{thm:SimpleTheorem} is simple but
the upper bound in~\eqref{eq:SimpleBound} depends on $\|X\|_2$, which typically grows rapidly with $n$.
However, a simple model for low rank matrices generated in data science
explains why we should expect these matrices to have a small spectral norm.
We suppose that $X$ is generated by sampling columns and rows from a
so-called nice latent variable model
(intuitively, smooth; see Definition~\ref{def:niceLVM} for a formal definition),
or a piecewise nice model. A nice latent variable model has a simple parametrization,
but not a \emph{linear} parametrization. One might worry that a low rank approximation
could overlook this structure, and a more complicated approximation scheme would be required to compress such datasets. However, our main theorem suggests that low rank approximation is a remarkably powerful technique for approximating datasets from nice latent variable models.
Hence this paper provides one explanation for the prevalence of low rank matrices in data science.
% as in kernel learning...?

Our main result can be informally summarized as follows: % \begin{theorem}[Informal statement]\label{thm-main-vague}
%Suppose $X \in \reals^{m \times n}$ is drawn from a nice latent variable model
%with entries of the form $X_{ij} = f(\alpha_i, \beta_j)$.
%Here, $f$ is a fixed function, and the $\alpha_i$s and $\beta_j$s are independent and identically distributed random variables.
%As the dimensions $m$ and $n$ of the matrix $X$ increase, the $\epsilon$-rank (See Definition~\ref{def:epsRank}) of  $X$ grows proportional to $\log(m+n)$.

\begin{center}
\emph{``Nice latent variables models are of log-rank.''}
\end{center}
After formally defining ``nice'' latent variable models and log-rank in Section~\ref{sec:background}, we state a precise version of this informal statement (see Theorem~\ref{thm-main-detailed}). Theorem~\ref{cor:piecewise} extends this result to piecewise nice latent variable models, while Theorem~\ref{cor:graphons} considers symmetric latent variable models, i.e., graphons.

Our main tool is the Johnson--Lindenstrauss Lemma (see Lemma~\ref{lem:Johnson}), which says that given any point cloud in a high-dimensional Euclidean space there exists an embedding onto a low dimensional Euclidean space
that approximately preserves pairwise distances between points.

This result has ramifications for how to interpret an underlying low rank structure in datasets.
In particular, we have good news for those designing algorithms: sufficiently large datasets tend to have low rank structure, which can be used to design faster algorithms.
Conversely, we have bad news for those who attempt to find meaning in low rank structure.
Researchers often give post-hoc explanations for why a particular dataset is approximately of low rank. For example, typical arguments are: customers' movie preferences are low rank because movies are well parametrized by a few meaningful \emph{genres} or that word document matrices are low rank because they are well parametrized by a handful of meaningful \emph{topics}.
Our main theorem shows that low rank structure can persist even without an underlying physical reason. In particular, a dataset from a nice latent variable model has an $\epsilon$-rank that grows slowly with its dimensions; \emph{no matter how many} genres or topics generate the data.

Throughout, we use $\|v\|^2 = \sum_{i=1}^Nv_i^2$ to denote the Euclidean length of a vector $v\in\mathbb{R}^{N}$ and $\|f\| = \sup_{x\in \Omega} \left|f(x)\right|$ to denote the supremum norm of $f:\Omega \to \mathbb{R}$ over its domain $\Omega$.

\section{Background material}\label{sec:background}

We review some necessary background material.

\subsection{Rank}
A nonzero matrix $X\in\mathbb{R}^{m\times n}$ is said to be of rank~1 if $X$ can be written as an outer-product of two column vectors, i.e., $X = u v^T$ for $u\in\mathbb{R}^{m\times 1}$ and $v\in\mathbb{R}^{n\times 1}$.
%, where $v^T$ is the conjugate transpose of $v$ if these were complex.
Moreover, a matrix $X$ is of rank $k$ if $k$ is the smallest integer so that $X$ can be written as a sum of $k$ rank $1$ matrices. That is,
\[
X = u_1v_1^T + \cdots + u_kv_k^T, \qquad u_1,\ldots,u_k\in\mathbb{R}^{m\times 1},\quad v_1,\ldots,v_k\in\mathbb{R}^{n\times 1}.
\]
Generically, a matrix is of full rank; however, we find in data science that a full rank matrix can often be well-approximated by a low rank matrix in the sense
that $X\approx  u_1v_1^T + \cdots + u_kv_k^T$. If one finds that a matrix $X$ can be well-approximated by a rank $k$ matrix, $X_k$, then one can perform
diagnostics directly on $X_k$, instead of $X$.

\subsection{The $\epsilon$-rank of a matrix}
A matrix $X$ can be approximated by a rank $k$ matrix, up to an absolute accuracy of $\epsilon>0$, if the $\epsilon$-rank of $X$ is less than equal to $k$.

\begin{definition}[$\epsilon$-rank]
Let $X\in\mathbb{R}^{m\times n}$ be a matrix and $\epsilon >0 $ a tolerance. The (absolute) $\epsilon$-rank of $X$ is given by
\[
{\rm rank}_\epsilon(X) = \min\left\{{\rm rank}(A) : A\in\mathbb{R}^{m\times n}, \text{ } \left\|X -  A\right\|_{\max} \leq \epsilon\right\},
\]
where $\|\cdot \|_{\max}$ is the absolute maximum matrix entry. That is, $k = {\rm rank}_\epsilon(X)$ is the smallest integer for which $X$ can be approximated by a rank $k$ matrix, up to an accuracy of $\epsilon$.
\label{def:epsRank}
\end{definition}

There are several alternative definitions of $\epsilon$-rank in the literature~\cite{beckermann2016singular}.

% For example if $\epsilon = 0.01$ and $k = {\rm rank}_\epsilon(X)$, then there is a rank $k$ matrix $X_k$ such that $X$ can be replaced by $X_k$ with error at most $0.01$ in each entry.

\subsection{A log-rank family of matrices}
We are interested in families of matrices $\mathcal{X} = \{X^{(m \times n)}\}_{m,n\geq 1}$, where the
$\epsilon$-rank of $X^{(m \times n)}\in\mathbb{R}^{m\times n}$ grows slower than a polylogarithm in $m$ and $n$.
We use the notation $X^{(m \times n)} \in \mathcal X$ to emphasize
that $X^{(m \times n)}$ is a matrix of size $m \times n$.

\begin{definition}\label{def:logRank}
An infinite family of matrices $\mathcal{X} = \{X^{(m \times n)}\}_{m,n\geq 1}$ is of log-rank if there is a polynomial $p$ such that for any fixed $\epsilon >0$,
\[
{\rm rank}_\epsilon(X^{(m \times n)}) = \mathcal{O}( p(\log(m+n))).
\]
\end{definition}
In many settings (including the results in this paper), the polynomial $p$ is simply $p(x) = x$.

In machine learning, $\mathcal{X}$ might represent a family of datasets.
One can generate datasets of varying dimensions
by sampling more examples (rows $m$) or features (columns $n$) from a data distribution:
say, by collecting the required number of text documents, patient records, customer preferences, or movie reviews.
A log-rank family of matrices contains datasets for which the $\epsilon$-rank
grows only slowly as we collect more examples and more features.  Low rank techniques often lead to algorithms that have near-optimal complexity for log-rank families of matrices.

%\subsection{Our class of matrices}
%In this paper, we will consider infinite families of matrices $\mathcal{X} = \{X^{(m \times n)}\}_{m,n\geq 1}$
%indexed by dimensions $(m,n) \in \mathbb{Z}^2$.
%Each matrix $X^{(m \times n)} \in \mathbb{R}^{m \times n}$ in the family is a (finite) $m \times n$ matrix.
%For example, in machine learning, $\mathcal{X}$ might represent a family of datasets;
%by sampling more examples (rows $m$) or features (columns $n$) from the data distribution,
%we can generate concrete datasets of varying sizes.
%Equivalently, we can imagine that each matrix $X^{(m \times n)}$ is generated by choosing
%a $m$ rows and $n$ columns from an matrix with infinitely many rows and columns.

\subsection{Latent variable models}\label{sec:LVM}
Latent variable models (LVMs) are a particularly interesting class of families of matrices.
A latent variable model is parametrized by a continuous function $f$
and two distributions $\mathcal A$ and $\mathcal B$.
A family of matrices $\mathcal{X}_{f, \mathcal A, \mathcal B} = \{X^{(m \times n)}\}_{m,n\geq 1}$ is a latent variable model
(depending on $f$, $\mathcal A$, and $\mathcal B$) if for every $X^{(m\times n)}\in\mathcal{X}$,
\[
(X^{(m \times n)})_{ij}=f(\alpha_i, \beta_j), \qquad 1\leq i\leq m, \quad 1\leq j\leq n,
\]
where $\alpha_i$ and $\beta_j$ are independent random variables from the distributions $\mathcal A$ and $\mathcal B$, respectively.

Latent variable models have a natural relationship to low rank matrices.
Let us consider two particular well-studied latent variable models to understand
how these models lead to low rank matrices.
\begin{itemize}[leftmargin=*,noitemsep]
  \item \emph{Inner products.}
  Suppose $\mathcal A$ and $\mathcal B$ are distributions over
  vectors in $\reals^r$, and $f(\alpha,\beta) = \alpha^T\beta$ is an inner product.
  Then the rank of any matrix in the family $\mathcal{X}_{f, \mathcal A, \mathcal B}$
  is at most $r$.
  Note that this bound is independent of the dimension of the matrix.
  \item \emph{Smooth scalar functions.}
  Suppose $\mathcal A$ and $\mathcal B$ are distributions over the interval $[-1,1]$,
  and $f: \reals \times \reals \to \reals$ is an analytic function with bounded derivatives:
  $\|f^{(k)}\| \leq M$ for every $k$.
  Then the $\epsilon$-rank of any matrix in the family $\mathcal{X}_{f, \mathcal A, \mathcal B}$
  is at most $\log(\frac{2M}{\epsilon})$.
  To see this, expand $f(\alpha, \beta)$ around $\beta = 0$ as
  \[
  f(\alpha, \beta) = \sum_{k=0}^\infty \frac 1 {k!} f^{(k)}(\alpha, 0)\beta^k.
  \]
  For any matrix $X$ in the family $\mathcal{X}_{f, \mathcal A, \mathcal B}$,
  we can truncate this expansion at the $K$th term to obtain a rank $K$ approximation to $X$.

  To understand the quality of this approximation, consider the tail sum
  \beas
  \sum_{k=K}^\infty \frac 1 {k!} f^{(k)}(\alpha, 0)\beta^k
  &\leq& \sum_{k=K}^\infty \frac M {k!} \leq \frac {2M} {K!}.
  \eeas
  Using Stirling's approximation $K! \geq \sqrt{2 \pi K} (\frac K e)^K$ \cite{robbins1955remark}, we see% this is *actually* a lower bound, which is the direction we need
  \beas
  \sum_{k=K}^\infty \frac 1 {k!} f^{(k)}(\alpha, 0)\beta^k
  \leq \frac {2M} {K!}
  \leq 2M \left (\frac e K \right)^K
  \leq 2M \left ( \frac 1 2 \right)^K \leq \epsilon
  \eeas
  if $K \geq \log(\frac {2M} \epsilon)$.
  %we see that an approximation with rank $K \leq \log(\frac {2M} \epsilon)$ suffices.
  Hence we see that the $\epsilon$-rank of any matrix in the family $\mathcal{X}_{f, \mathcal A, \mathcal B}$
  is at most $\log(\frac {2M} \epsilon)$.
  Note that this bound is again independent of the dimension of the matrix.
  \item \emph{Smooth vector functions.}
  The previous argument used a Taylor expansion of the function in the parameter $\beta$.
  If $\mathcal A$ and $\mathcal B$ are both distributions over a bounded set in $\reals^N$,
  and $f: \reals \times \reals \to \reals$ is an analytic function with bounded derivatives,
  we can use the same argument to expand $f$ in the vector $\beta$
  to again obtain a bound on the $\epsilon$-rank independent of the dimension of the matrix.

  The bound again depends logarithmically on $\frac 1 \epsilon$;
  however, the bound grows exponentially in the dimension $N$ of the latent variables.
  See Lemma~\ref{lem-bounded-rank} for the formal argument.
  Our main result, Theorem~\ref{thm-main-detailed} eliminates the dependence
  on the dimension $N$ of the latent variables by introducing
  a dependence on the dimension of the matrix.
\end{itemize}
Latent variable models can also be used to model more complex distributions.
For example, $f$ might be a kernel function, and $\mathcal A$ and $\mathcal B$ might be distributions over very high-dimensional spaces.
% We survey a few interesting latent variable models in Section~ref{s-lvm}.
% In data analysis and machine learning, latent variable models often arise as XXX.

\subsection{The Johnson--Lindenstrauss Lemma}
%We say that an infinite sequence of random variables $\mathcal{Z} = (Z_i)_{i\geq 1}$ is exchangeable if for each $n$ and permutation $\pi$ on $\mathbb{N}$ we have
%\[
%(Z_1,\ldots,Z_n) \sim (Z_{\pi(1)},\ldots,Z_{\pi(n)}),
%\]
%where `$\sim$' denotes equality in distribution.
%That is, conditional on the tail sigma field the variables $Z_1,Z_2,\ldots,$ are independent and identically distributed. Aldous in~\cite{Aldous_81_01}
%extends this notion of exchangeable to matrices.
%
%\begin{definition} An infinite matrix of random variables $\mathcal{X}$ is exchangeable if its sequence of rows $\mathcal{R} = \left(\mathcal{X}_{i,:}\right)_{i\geq 1}$ and
%columns $\mathcal{C} = \left(\mathcal{X}_{:,j}\right)_{j\geq 1}$ form exchangeable sequences. Equivalently, for each $m$ and $n$ and
%permutations $\pi$ and $\tau$ on $\mathbb{N}$ we have
%\[
%\begin{bmatrix}
%\mathcal{X}_{11} & \ldots & \mathcal{X}_{1n}\\ \vdots & \ddots &\vdots \\ \mathcal{X}_{m1} &\ldots& \mathcal{X}_{mn}
%\end{bmatrix}
%\sim
%\begin{bmatrix}
%\mathcal{X}_{\pi(1)\tau(1)}& \ldots &\mathcal{X}_{\pi(1)\tau(n)}\\ \vdots & \ddots &\vdots \\ \mathcal{X}_{\pi(m)\tau(1)}& \ldots &\mathcal{X}_{\pi(m)\tau(n)}
%\end{bmatrix}\!,
%\]
%where `$\sim$' denotes entry-by-entry equality in distribution.
%\end{definition}

% Szameredi's regularity lemma has the following consequence: (?)
A key tool in theoretical computer science is the Johnson--Lindenstrauss Lemma~\cite{Johnson_84_01}.  Roughly, it
says that a high dimensional point cloud can be projected onto a low-dimensional space while approximately preserving all pairwise distances between the points.  There are several alternative forms and proofs~\cite{matouvsek2008variants}.

\begin{lemma}[The Johnson--Lindenstrauss Lemma]
Let $0<\epsilon_\text{JL}<1$, $x_1,\ldots,x_n$ be $n$ points in $\mathbb{R}^{N}$, and $r = \lceil 8(\log n)/\epsilon_\text{JL}^2\rceil$. Then,
there is a linear map $Q: \mathbb{R}^N \rightarrow \mathbb{R}^r$ such that
\[
(1-\epsilon_\text{JL} )\|x_i-x_j\|^{2}
\leq \|Q(x_i-x_j)\|^{2}
\leq (1+\epsilon_\text{JL} )\|x_i-x_j\|^{2}, \qquad 1\leq i,j \leq n.
\]
Here, $\lceil a\rceil$ is the smallest integer larger than $a$.
\label{lem:Johnson}
\end{lemma}
\begin{proof}
See~\cite[Thm.~1.1]{matouvsek2008variants}. Also, see~\cite{Johnson_84_01}.
\end{proof}

A slight reformulation of the Johnson--Lindenstrauss Lemma is useful for us, which roughly says that a high-dimensional point cloud can be projected onto a low-dimensional space while approximately preserving inner-products between vectors.
\begin{lemma}[Variant of the Johnson--Lindenstrauss Lemma]
Let $0<\epsilon_\text{JL}<1$, $x_1,\ldots,x_n$ be $n$ points in $\mathbb{R}^{N}$, and $r = \lceil 8\log(n+1)/\epsilon_\text{JL}^2\rceil$. Then,
there is a linear map $Q: \mathbb{R}^N \rightarrow \mathbb{R}^r$ such that
\[
\left| x_i^Tx_j - x_i^TQ^TQx_j \right| \leq \epsilon_{\text{JL}}\left(\|x_i\|^2 + \|x_j\|^2-x_j^Tx_k\right), \qquad 1\leq i,j\leq n.
\]
\label{lem:Johnson2}
\end{lemma}
\begin{proof}
Consider the point set $\{x_1,\ldots,x_n,0\}\subset \mathbb{R}^N$.  Since $r = \lceil 8(\log(n+1))/\epsilon_\text{JL}^2\rceil$, the Johnson--Lindenstrauss Lemma says that there exists a linear map $Q: \mathbb{R}^N \rightarrow \mathbb{R}^r$ such that
\[
\begin{aligned}
(1-\epsilon_\text{JL} )\|x_i\|^{2} \leq & \|Q x_i \|^{2} \leq (1+\epsilon_\text{JL} )\|x_i\|^{2}, & 1\leq i\leq n,\\
(1-\epsilon_\text{JL} )\|x_i-x_j\|^{2} \leq & \|Q(x_i-x_j) \|^{2} \leq (1+\epsilon_\text{JL} )\|x_i-x_j\|^{2}, \qquad &1\leq i,j\leq n.\\
\end{aligned}
\]
Therefore, from the identity $2a^Tb = \|a\|^2 + \|b\|^2 - \|b-a\|^2$ we find that
\[
\begin{aligned}
(1-\epsilon_{\text{JL}})(\|x_j\|^2 + \|x_k\|^2) &- (1+\epsilon_{\text{JL}})\|x_j-x_k\|^2 \leq 2x_j^TQ^TQx_k\\
&\qquad \leq (1+\epsilon_{\text{JL}})(\|x_j\|^2 + \|x_k\|^2) - (1-\epsilon_{\text{JL}})\|x_j-x_k\|^2.
\end{aligned}
\]
Using the identity $2a^Tb = \|a\|^2 + \|b\|^2 - \|b-a\|^2$ again, we obtain
\[
-\epsilon_{\text{JL}}(\|x_j\|^2 + \|x_k\|^2 -x_j^Tx_k)\leq x_j^Tx_k - x_j^TQ^TQx_k \leq \epsilon_{\text{JL}}(\|x_j\|^2 + \|x_k\|^2 -x_j^Tx_k),
\]
as required.
\end{proof}

\subsection{Extremely large matrices are low rank in the max norm}
The variant of the Johnson--Lindenstrauss Lemma in Lemma~\ref{lem:Johnson2} allows us to prove Theorem~\ref{thm:SimpleTheorem}.
\begin{proof}[Proof of Theorem~\ref{thm:SimpleTheorem}]
The singular value decomposition of $X$ is $X = U\Sigma V^T$.  We can write $X = \tilde{U}\tilde{V}^T$, where $\tilde{U} = U\sqrt{\Sigma}$ and $\tilde{V} = V\sqrt{\Sigma}$. Applying Lemma~\ref{lem:Johnson2} with $\epsilon_{JL} = \epsilon/3$ to the set $\left\{\tilde{u}_1,\ldots,\tilde{u}_n,\tilde{v}_1,\ldots\tilde{v}_n,0\right\}$ with $\tilde{u}_j$ and $\tilde{v}_j$ being the $j$th column of $\tilde{U}$ and $\tilde{V}$, respectively, we find that for $r = \lceil 72 \log (2n+1) / \epsilon^2 \rceil$ there exists a $Q\in\mathbb{R}^{n\times r}$ such that
\[
\left| \tilde{u}_i^T \tilde{v}_j -  \tilde{u}_i^TQ^TQ\tilde{v}_j\right| \leq \epsilon_{JL} \left(\|\tilde{u}_i\|^2 + \|\tilde{v}_j\|^2 - \tilde{u}_i^T\tilde{v}_j\right).
%& \epsilon_{JL} \left( 2\|X\|_2 + \tilde{u}_i^T\tilde{v}_j\right) \\
\]
Since $X_{ij} = \tilde{u}_i^T\tilde{v}_j$, $\|\tilde{u}_i\|^2 = \sigma_i(X) \leq \|X\|_2$, and $\|\tilde{v}_i\|^2 = \sigma_i(X) \leq \|X\|_2$, we find that
\[
\begin{aligned}
\left| X_{jk} -  \tilde{u}_i^TQ^TQ\tilde{v}_j\right| & \leq \epsilon_{JL} \left(2\|X\|_2 + \|X\|_{\text{max}} \right)\\
& \leq 3\epsilon_{JL} \|X\|_2,
\end{aligned}
\]
where the last inequality uses the fact that $\|X\|_{\text{max}} \leq \|X\|_2$. The result follows by setting $Y_{ij} =  \tilde{u}_i^TQ^TQ\tilde{v}_j$ and noting that $\epsilon = 3\epsilon_{JL}$.
\end{proof}

\section{Related work} % conventional section header for NIPS on the $\epsilon$-rank of a matrix
The majority of the literature focuses on either how to find low rank matrices
or how to exploit low rank structure after it has been found.
This trend is set to continue with the emerging field of multilinear algebra,
and the increasing use of tensor factorizations in machine learning and data analysis~\cite{omberg2007tensor,ho2014marble,ho2014limestone}.
This keen practical interest in low rank structure lends urgency to the quest to understand why and when low rank techniques work well on real datasets.

%Furthermore, many researchers work on finding good low rank approximations,
%either from complete or partially observed matrices, assuming an underlying  the matrix is low rank, approximately low rank,
%or low complexity (eg nuclear norm).

\subsection{Bounds on $\epsilon$-rank.}
The work of Alon and his coauthors is closest in spirit to our paper~\cite{alon2009perturbed,alon2013approximate}.
These papers use the Johnson--Lindenstrauss Lemma to show that the identity matrix, and any
positive semidefinite matrix, has an $\epsilon$-rank that grows logarithmically with
the number of columns and rows.
% Here, we also use the Johnson--Lindenstrauss Lemma to show that
% any nice latent variable model also have an $\epsilon$-rank that grows logarithmically with dimension.
% who defines the epsilon-rank as
% \[
% \epsilon-\text{rank}(A) = \min\{\rank(B): \|B - A\|_\infty \leq \epsilon,
% \]
% and shows that the epsilon-rank of the identity
% is at least $\Omega(\log n / \epsilon^2 \log(1/\epsilon))$ and
% at most $\mathcal O(\log n / \epsilon^2)$.
% In a later work, Alon \etal \cite{alon2013approximate} show that
% any symmetric positive semidefinite matrix $A$ with entries $|X_{ij}|\leq 1$
% has
% \[
% \epsilon-\text{rank}(A) \leq \frac{9 \log n}{\epsilon^2 - \epsilon^3}.
% \]

Chatterjee shows that any matrix with bounded entries can be well-approximated by
thresholding all singular values lower than a given value to $0$~\cite{Chatterjee_15_01}.
His main theorem implies that the $\epsilon$-rank
of a matrix of size $n \times n$ grows like $\mathcal O(\sqrt{n})$.
Our theorem improves this result to $\mathcal O(\log n )$
when the matrix comes from a nice latent variable model.

In~\cite{beckermann2016singular}, bounds were derived on a slightly different $\epsilon$-rank of certain matrices $X\in\mathbb{R}^{m\times n}$ with
displacement structure, i.e., a matrix that satisfies $AX-XB = F$. For example,~\cite[Thm.~3.1]{beckermann2016singular}
showed that all $n\times n$ positive-definite Hankel matrices, $(H_n)_{ij} = h_{i+j}$, have an $\epsilon$-rank that grows logarithmically in $n$. These results were later extended to include a broader class of matrices~\cite{townsend2017singular}. These results from linear algebra are considering matrices that have more rapidly decaying singular values than the LVMs we study in this paper.

\subsection{Exchangeable families of matrices.}
Latent variable models are related to so-called {\em exchangeable} families of matrices.  We say that an infinite matrix $\mathcal{X}$ is exchangeable if for any permutations $\sigma$ and $\pi$ on $\mathbb{N}$, we have
\[
\mathcal{X}_{i,j} \sim \mathcal{X}_{\sigma(i),\pi(j)},\qquad 1\leq i\leq m, \quad 1\leq j\leq n,
\]
where `$\sim$' denotes equality in distribution.
A celebrated result by Aldous~\cite{Aldous_81_01} states that if $\mathcal{X}$ is exchangeable, then
\[
\mathcal{X}_{ij} \sim f(\omega,\alpha_i, \beta_j, \eta_{ij}),
\]
where $f$ is a measurable function, $\omega$, $\alpha_i,~\beta_j,~\eta_{ij}$ are scalar-valued, and the $\omega$, $\alpha_i$s, $\beta_j$s, and $\eta_{ij}$s are mutually independent and uniformly distributed random variables on $[0,1]$. One can generate a family of matrices from $\mathcal{X}$ by taking the leading $m\times n$ principal submatrices.

There is some resemblance here to the latent variable model. There are two significant differences: (1) There is an intrinsic noise term $\eta_{ij}$ and (2)
The latent variables $\omega$, $\alpha_i$, and $\beta_j$ are scalar-valued and uniform random variables on $[0,1]$.  Our result on latent variable models can be extended to exchangeable families of matrices, under additional smoothness assumptions on $f$.

%Rasch model. Row-column-exchangeable doubly infinite matrix.

The symmetric analogue of an exchangeable array is a graphon.
Graphons can be seen as the continuous limit of a sequence of (dense) graphs~\cite{lovasz2006limits}.
Many authors have proposed methods for graphon estimation from samples of the entries~\cite{choi2012stochastic,airoldi2013stochastic,wolfe2013nonparametric,chan2014consistent}.
% This theory is generally concerned with the statistical consistency of the estimator,
% rather than its parsimony.
For example, Airoldi \etal~required that the graphon be piecewise
Lipshitz, and provided an approximate graphon that gives a complexity that grows linearly in the number of pieces~\cite{airoldi2013stochastic}.
Our theory shows that this procedure overestimates the complexity required to model a graphon
when the graphon is nice. Indeed, Theorem~\ref{cor:piecewise} shows that
the $\epsilon$-rank of a nice graphon grows
%like the \emph{logarithm} of the dimension of the graph, and
with the maximum complexity of each piece. For reasonable distributions, the maximum complexity grows sublinearly in the number of pieces.
Choi et al.~showed that it is possible to find a consistent estimator for the graphon when the number
of classes in a stochastic block model grows at most like the square root of the dimension~\cite{choi2012stochastic}.
Our theory shows that a low rank model for the graphon (which generalizes a stochastic block model)
only requires a rank that grows like the logarithm of the dimension.
Whether it is possible to find statistically consistent estimators that obtain this threshold is an important question for future research.

The theory of exchangeable matrices has been used to motivate
the use of latent variable models for collaborative filtering and other applications in
machine learning. For example, many authors have used the assumption that
the latent variable model is Lipschitz to design efficient estimators for
symmetric and asymmetric distributions of data~\cite{song2016blind,lee2017unifying,lee2017blind}.
We show a connection between this approach and the standard low rank model.
%eg Christina Lee.

\section{Any nice latent variable model is log-rank}
%In what follows, we use $\|\cdot\|_2$ to denote
%the Euclidean norm of a vector and $\|f\| = \sup_x |f(x)|$ denotes the supremum  norm on continuous bounded functions.

Our result applies to any {\em nice} latent variable model, which we now define.
\begin{definition}\label{def:niceLVM}
  A latent variable model $\mathcal{X} = \mathcal{X}_{f, \mathcal A, \mathcal B}$ is called
  \emph{nice} with parameters $(N,R,C,M)$ if the following conditions hold:
  \begin{itemize}[leftmargin=*,noitemsep]
  \item The associated distributions $\mathcal A$ and $\mathcal B$ are supported on a closed ball $B_R \subset \reals^N$
  for some $N\geq 1$ of radius $R>0$, i.e., $B_R = \{x\in\reals^N:\|x\|\leq R\}$. Here, $N$ is allowed to be extremely large.
  \item The associated function $f: B_R \times B_R\rightarrow\reals$ is bounded and sufficiently smooth in the sense that $f(\alpha,\cdot)$ is uniformly analytic in $B_R$ for every $\alpha\in B_R$ and for all $\mu\in\mathbb{N}^N$ we have
  \[
  %\sup_{\alpha,\beta \in B_R} \left|D^\mu f(\alpha, \beta)\right| \leq CM^{|\mu|}\|f\|.
  \left\|D^\mu f(\alpha, \beta)\right\| \leq CM^{|\mu|}\|f\|.
  \]
Here, $\mu = (\mu_1,\ldots,\mu_N)$ is a multi-index, $|\mu| = \sum_{i=1}^N \mu_i$, $D^\mu f = \frac{\partial^{|\mu|}}{\partial^{\mu_1}\! \beta_1\cdots \partial^{\mu_N}\! \beta_N}$, and $C\geq 0$ and $M\geq 0$ are positive constants.
  \end{itemize}
\end{definition}
% It's ok if they live in an infinite dimensional space, too, as long
% as the distribution $\mathcal B$ has finite variance, but the notation is more cumbersome.

Nice latent variable models are common in machine learning and data analysis.
Functions that give rise to nice latent variable models include:
\begin{itemize}[leftmargin=*,noitemsep]
\item \emph{Linear functions.} If $f(\alpha, \beta) = \alpha^T \beta$ and the distributions $\mathcal{A}$ and $\mathcal{B}$ have bounded support, then $\mathcal{X}_{f,\mathcal{A},\mathcal{B}}$ is a nice LVM with $M=C=1$. In this case, $\mathcal{X}_{f, \mathcal A, \mathcal B}$ has a rank bounded by $N$. Theorem~\ref{thm-main-detailed} shows that when $N$ is sufficiently large the $\epsilon$-rank is actually smaller than $N$ for $\epsilon>0$.
\item \emph{Polynomials.} If $f$ is a polynomial in $2N$-variables, then there is a constant $M$ that depends on $N$, $R$, and the degree of the polynomial so that $\left\|D^\mu f(\alpha, \beta)\right\| \leq CM^{|\mu|}\|f\|$. For simplicity, consider $N=1$ and $f(\alpha,\beta) = \beta^d$. Then, for $k<d$ we have
\[
\|D^{k} f(\alpha,\beta)\| = d(d-1)\cdots(d-k+1)\sup_{|\beta|\leq R} |\beta|^{d-k} \leq d^k R^{-k} \|f\|.
\]
So, $M = d/R$ and $C = 1$ suffices.
\item \emph{Kernels.} If $f(\alpha, \beta) = e^{p(\alpha, \beta)}$ for a $2N$-variable polynomial $p$,
then $\left\|D^\mu f(\alpha, \beta)\right\| \leq CM^{|\mu|}\|f\|$ for some constants $C$ and $M$. This includes most kernels typically used in machine learning.
For example, consider the radial basis function kernel
$f(\alpha, \beta) = \exp(-\|\alpha - \beta\|^2)$ with $R>1/2$.
Then, $\left\|D^\mu f(\alpha, \beta)\right\| \leq N (4R)^{N+|\mu|}\|f\|$.
% \item \emph{Tetration.} We include this example to show how quickly $f$ can grow while
% remaining nice.
% If $f$ is an exponential tower, such as $f(\alpha,\beta) = e^{e^{p(\alpha,\beta)}}$,
% $M$ can be taken to be \mutodo{fix} $1 + \|e^x\|$.
\end{itemize}
% Not nice example: 1/x, even in a ball around x=1.
% Not nice example: exp(-1/x^2) (smooth but not analytic)
We see that the bound on the derivatives of $f$ allows for many relevant examples.
Our framework can also handle the case of piecewise nice LVMs, which we treat below
in Theorem~\ref{cor:piecewise}.

% We don't use this
% For use later on, we define the norm of a LVM as
% \[
% \|\mathcal{X}_{f, \mathcal A, \mathcal B}\| =
% \sup_{\alpha \in \mathcal A, \beta \in \mathcal B} | f(\alpha, \beta) |.
% \]
% Therefore, we have $\|X\|_{\max}\leq \|\mathcal{X}_{f, \mathcal A, \mathcal B}\|$ for any $X \in \mathcal{X}_{f, \mathcal A, \mathcal B}$.

We are now ready to formally state our main result.
An alternative theorem with the analytic assumptions of $f$ on the first variable
is also possible with an analogous proof.
% We already have enough foreshadowing: two forward references above.
% As a consequence of our main theorem, we relax the requirement
% that $f$ be analytic everywhere in its second variable (see Theorem~\ref{cor:piecewise}).

\begin{theorem}\label{thm-main-detailed}
Let $\mathcal{X}_{f, \mathcal A, \mathcal B}$ be a nice latent variable model and $0<\epsilon<1$.
Then, for each $X^{(m\times n)} \in \mathcal X_{f, \mathcal A, \mathcal B}$, the $\epsilon \|f\|$-rank of
$X^{(m\times n)}$ is no more than
%r = 8 \log(m+n+1) \left(\frac{C_u + C_v + 1}{\epsilon} + \frac 1 2 \right)^2,
\[
r = \Bigg\lceil 8 \log(m+n+1) \left(1 + \frac{2(C_u + C_v + 1)}{\epsilon} \right)^2\Bigg\rceil,
\]
where $C_u$ and $C_v$ are constants defined below that depend on
the latent variable model $\mathcal{X}_{f, \mathcal A, \mathcal B}$.
\end{theorem}

We state Theorem~\ref{thm-main-detailed} in terms of the $\epsilon \|f\|$-rank to show that we achieve
a natural sort of relative-error guarantee.
Consider the LVM $\mathcal{X'}_{f', \mathcal A, \mathcal B}$ where $f' = cf$ for some constant $c$.
The entries of a matrix drawn from $\mathcal{X'}_{f', \mathcal A, \mathcal B}$ are about a factor of $c$
larger in expectation than the entries of a matrix drawn from $\mathcal{X}_{f, \mathcal A, \mathcal B}$.
It is natural to compare the $c$-rank of a matrix from $\mathcal{X'}_{f', \mathcal A, \mathcal B}$ with the $1$-rank of a matrix from $\mathcal{X}_{f, \mathcal A, \mathcal B}$.
Theorem~\ref{thm-main-detailed} shows both satisfy the same bound, since $\|f'\| = c\|f\|$.

The proof proceeds in two main steps.
The first is to find an explicit (possibly high) rank factorization of some approximation
$\hat X$ to a matrix $X^{(m\times n)} \in \mathcal X_{f, \mathcal A, \mathcal B}$
drawn from the latent variable model.
We use a Taylor expansion of the function $f(\alpha,\cdot)$ about $0$ to show that $f(\alpha_i, \beta_j) \approx u_i^T v_j$. That is, $f$ can be well-approximated as
the inner product between two (high dimensional) vectors,
$u_i$ and $v_j$, with bounded Euclidean norms.
The second step is to use the Johnson--Lindenstrauss Lemma to reduce the dimensionality of
the set of vectors $\{0,u_1,\ldots,u_m,v_1,\ldots,v_n\}$ while
approximating preserving the inner products $u_i^Tv_j$. % for $1\leq i\leq m$, $1\leq j\leq n$.

We present the first step as a lemma.

\begin{lemma}[Bounded rank approximation]\label{lem-bounded-rank}
Let $\mathcal{X}_{f, \mathcal A, \mathcal B}$ be a nice latent variable model
with parameters $(N,R,C,M)$ and let $0<\epsilon<1$.
Then, for each $X^{(m\times n)} \in \mathcal X_{f, \mathcal A, \mathcal B}$,
there is some $\epsilon$-approximation $\hat X \in \reals^{m \times n}$
with $\|X - \hat X\|_\text{max} \leq \epsilon \|f\|$ and
% equivalently, $\epsilon \|f\|$-rank of $X^{(m\times n)}$ is no more than
%r = 8 \log(m+n+1) \left(\frac{C_u + C_v + 1}{\epsilon} + \frac 1 2 \right)^2,
\[
\rank(\hat{X}) \leq (K+1)N^K \qquad \text{where } K \leq \max(2e^1 NRM, \log_2(C/\epsilon)) + 1.
\]

Furthermore, $\hat X$ admits a rank $\tilde N \leq (K+1)N^K$ factorization as
\[
\hat X_{ij} = u_i^T v_j \quad 1\leq i\leq m, \quad 1\leq j\leq n,
\]
where each $u_i \in \reals^{\tilde{N}}$ and $v_j \in \reals^{\tilde{N}}$ obey
\[
\|u_i\| \leq C_u \|f\|, \qquad \|v_j\| \leq C_v \|f\|.
\]
Here, $C_u$ and $C_v$ are constants depending on the latent variable model
$\mathcal X_{f, \mathcal A, \mathcal B}$
but not on the dimensions $m$ or $n$.
\end{lemma}

Notice that the vectors $u_i$ and $v_j$ may have an extremely large number of entries
when the dimension $N$ of the latent variable model is large:
this bound on the rank of $\hat X$ grows as $N^N$.

\begin{proof}[Proof of Lemma~\ref{lem-bounded-rank}]
%Without loss of generality, suppose that
%$\beta$ is supported on a ball $B_R$ centered at the origin of radius $R < 1$;
%otherwise, perform a change-of-variables on $f$ to ensure this.
%Since the original LVM is nice and the $\mu$-derivatives of $f$ are bounded by $CM^{|\mu|}\|f\|$, the
%LVM remains nice after this change.

We'll begin by showing that $f(\alpha_i, \beta_j) \approx u_i^T v_j$.
By Taylor expanding $f(\ualpha_i,\ubeta_j)$ in the second variable about $0$ with $K$
terms, we find that
\[
\left|X_{ij} -\hat{X}_{ij}\right| \leq
\frac{N^{K+1} R^{K+1}}{(K+1)!}\max_{|\tau|=K+1} \sup_{z\in B_R} \left|D^{\tau} f(\ualpha_i,z)\right|,
\quad
\hat{X}_{ij} =  \sum_{|\mu|\leq K}\frac{D^\mu f(\ualpha_i,0)}{\mu!}\ubeta_j^\mu,
\]
where $D^\mu f = \frac{\partial^{|\mu|}}{\partial^{\mu_1} \beta_1\cdots \partial^{\mu_N} \beta_N}$,
$\mu! = \mu_1! \cdots \mu_N!$, and
$\ubeta_j^\mu = (\beta_j)_1^{\mu_1}\cdots(\beta_j)_N^{\mu_N}$. Here,
the $N^{K+1}$ term in the Taylor error comes from the fact that
there are fewer than $N^{K+1}$ $\mu$'s with $|\mu| = K+1$: to get a term with $|\mu| = K+1$,
we must choose $K+1$ elements from the $N$ coordinates (with replacement).

From the formula for $\hat{X}_{ij}$, there are vectors $u_i$ and $v_j$ with
$\tilde{N} := \sum_{|\mu|\leq K} 1$ entries,
such that $\hat{X}_{ij} = u_i^Tv_j$.
From the simple counting argument above, we can see
\[
\tilde{N} = \sum_{|\mu|\leq K} 1 = \sum_{k=0}^K \sum_{|\mu|=K} 1 \leq \sum_{k=0}^K N^k \leq (K+1)N^K.
\]
The vectors $u_i$ and $v_j$ are indexed by $|\mu| \leq K$ and can be taken to be
\[
(u_i)_\mu = \frac 1 {\sqrt{\mu!} \sqrt{\|f\|}} D^\mu f(\alpha_i, 0),
\qquad
(v_j)_\mu = \frac{1}{\sqrt{\mu!}}\sqrt{\|f\|} \beta_j^{\mu}.
\]
Hence, we write
\[
\hat{X} = UV, \qquad U = \begin{bmatrix} u_1 | \cdots | u_m\end{bmatrix}^T, \quad V = \begin{bmatrix} v_1 | \cdots | v_n\end{bmatrix}.
\]
This result immediately gives a bound on the rank of $\hat{X}$.
For example, if $N=1$, we have $\rank(\hat{X}) \leq \tilde{N} \leq (K+1)N^K = K+1$.
% However, this bound on the rank grows very quickly with the dimension $N$ of the latent variables.

% \mutodo{Can we compute these? Also, can we allow derivatives of $f$ to grow exponentially (but not superexponentially)? How about this: wlog, let $R \leq 1$, so that $\|v\|^2 < 2$.
% $u_\mu^2 = (\frac 1 {\sqrt \mu!} D^\mu_\beta f(\alpha, 0))^2 \leq
% \frac 1 {\mu!} f(\alpha, 0)^2 p(|\mu|)^2
% \leq \|f\|^2 \frac 1 {\mu!} p(|\mu|)^2 = C \|f\|^2$, (where $C$ is some constant,
% so it's proportional to $\|f\|^2$.
% A less sloppy analysis should get rid of the square, by splitting the magnitude
% equally between $u$ and $v$.
% }

Now, select $K$ sufficiently large so that
\beas
\left|X_{ij} -\hat{X}_{ij}\right|
&\leq&
\frac{N^{K+1} R^{K+1}}{(K+1)!}\max_{|\tau|=K+1} \sup_{z\in B_R} \left|D^{\tau} f(\ualpha_i,z)\right| \\
&\leq&
C\frac{N^{K+1}R^{K+1}M^{K+1}}{(K+1)!} \|f\| \\
&\leq & \epsilon \|f\|.
\eeas
Since the denominator grows superexponentially in $K$,
there is always a sufficiently large $K$ for the bound above for any $0<\epsilon <1$.

To find a concrete bound on $K$, let us use Stirling's formula: $K! \geq \sqrt{2 \pi K} (\frac K e)^K$ \cite{robbins1955remark}.
Pick $K \geq 2eNRM$, so $\frac{eNRM}{K+1} \leq \frac 1 2$.
Substituting Stirling's formula into the previous display, we see
\beas
\left|X_{ij} -\hat{X}_{ij}\right|
&\leq&
\frac{1}{\sqrt{2 \pi (K+1)}} \left(\frac {NRMe} {K+1} \right)^{K+1} C \|f\| \\
&\leq&
\left(\frac {NRMe} {K+1} \right)^{K+1} C \|f\| \\
&\leq&
\left(\frac 1 2\right)^{K+1} C \|f\| \\
&\leq&
\epsilon \|f\|
\eeas
if $K \geq \log_2 (\nicefrac{C}{\epsilon})$.
Hence $K \geq \max(2eNRM, \log_2 (\nicefrac{C}{\epsilon}))$ suffices to achieve a
$\epsilon \|f\|$-approximation to $X$.

Therefore, we have the approximation
\[
|X_{ij} - \hat X_{ij}| \leq \epsilon \|f\|, \quad \hat X_{ij} = u_i^T v_j \quad 1\leq i\leq m, \quad 1\leq j\leq n,
\]
where $u_i \in \reals^{\tilde{N}}$ and $v_j \in \reals^{\tilde{N}}$ for $1\leq i\leq m$ and $1\leq j\leq n$.

Let us remark on the norms of $u_i$ and of $v_j$.
%for any $i=1,\ldots,m$ and $j = 1,\ldots,n$.
We suppress the indices $i$ and $j$ in this discussion.

Let $u^{(\infty)} = (u_\mu)_{|\mu| \geq 0}$ and
$v^{(\infty)} = (v_\mu)_{|\mu| \geq 0}$ be infinite dimensional vectors.
Then,
\[
\|u\|^2\leq \|u^{(\infty)}\|^2 = C_u \|f\| < \infty, \quad
\|v\|^2 \leq \|v^{(\infty)}\|^2 = C_v \|f\| < \infty,
\]
where $C_u$ and $C_v$ are constants that depend only on the properties of the nice LVM.

For $C_v$ we have
\[
\|v^{(\infty)}\|^2 \leq \sum_{\mu} \frac{1}{\mu!} \left|\beta^{2\mu}\right|\|f\| \leq \sum_{s=0}^\infty \frac{1}{s!}(N+s)^{N} R^{2s}\|f\| \leq C_v\|f\|,
\]
showing that $C_v$ is finite.

The constant $C_u$ depends on how quickly the derivatives of $f$ grow;
it is bounded so long as they grow no faster than exponentially.
Since $\mu!\geq(\lfloor |\mu|/N\rfloor )!$, we have
\[
|u_\mu|^2 = \frac 1 {\mu! \|f\|} |D^\mu f(\alpha, 0)|^2 \leq
C^2 M^{2|\mu|} \|f\| \frac{1}{(\lfloor |\mu|/N\rfloor )!}.
\]
Hence, we see that
\[
\|u^{(\infty)}\|^2 \leq \sum_{s = 0}^\infty (N+s)^{N}\frac{C^2M^{2s}}{(\lfloor s/N\rfloor)!} \|f\| \leq  C_u \|f\|,
\]
showing that $C_u$ is finite.
%\todo{Can we find explicit bounds on $C_u$ and $C_v$?}
\end{proof}
% Set $\epsilon_\text{JL}>0$ and $r = \lceil 8 (\log n) /\epsilon_\text{JL}^2\rceil$.
% We will approximate $\hat{A}$ by a rank $r$ matrix $\hat{A}_{r}$.
% Without loss of generality, we will suppose $m \geq n$,
% so the matrix has at least as many rows as it has columns.

We are now ready to prove our main theorem.

\begin{proof}[Proof of Theorem~\ref{thm-main-detailed}]
Suppose $X \in \mathcal X_{f, \mathcal A, \mathcal B} \cap \mathbb{R}^{m\times n}$
has entries $X_{ij} = f(\alpha_i,\beta_j)$ for each $1\leq i\leq m$ and $1\leq j\leq n$.

The proof proceeds in two steps.
First, use Lemma~\ref{lem-bounded-rank} to show that for each $1\leq i\leq m$ and $1\leq j\leq n$,
$|f(\alpha_i, \beta_j) - u_i^T v_j| \leq \epsilon / 2$
for two (extremely high dimensional) vectors,
$u_i \in \mathbb{R}^{\tilde N}$ and $v_j \in \mathbb{R}^{\tilde N}$, with Euclidean norms
bounded by $C_u \|f\|$ and $C_v \|f\|$, respectively.
Second, we use the Johnson--Lindenstrauss Lemma to
% reduce the dimensionality of $\{u_1,\ldots,u_m,v_1,\ldots,v_n\}$ while approximately preserving the inner products $u_i^Tv_j$. % for $1\leq i\leq m$, $1\leq j\leq n$.
show that $u_i^Tv_j \approx (Qu_i)^T Qv_j$ for $Q\in\mathbb{R}^{r\times \tilde{N}}$.
Let $r = \lceil8(\log(m+n+1)/\epsilon_{\text{JL}}^2\rceil$.
Then, by Lemma~\ref{lem:Johnson2} we know that there exists a linear map $Q \in \reals^{r \times \tilde{N}}$ such that
\[
|u_i^T v_j - u_i^TQ^TQv_j| \leq \epsilon_\text{JL} (\|u_i\|^2 + \|v_j\|^2 - u_i^T v_j), \qquad 1\leq i \leq m, \quad 1\leq j\leq n.
\]
Now, using our bound on $\|u_i\|^2$ and $\|v_j\|^2$ from above, we obtain the following inequalities for every $u \in \{u_1,\ldots,u_m\}$ and $v \in \{v_1,\ldots,v_n\}$:
\beas
|ui^T v - u^TQ^TQv|
&\leq& \epsilon_\text{JL} \left(\|u^{(\infty)}\|^2 + \|v^{(\infty)}\|^2 + |f(\alpha, \beta)| +\tfrac{\epsilon}{2}\|f\|\right) \\
&\leq& \epsilon_\text{JL} ((C_u + C_v)\|f\| + (1+\tfrac{\epsilon}{2}) \|f\|),
\eeas
where we have used the fact that $|f(\alpha, \beta)|\leq \|f\|$ and
$|u^T v - f(\alpha,\beta)| \leq \epsilon/2 \|f\|$.

The total error in each entry of our approximation is thus
\beas
|f(\alpha_i, \beta_j) - x_i^T y_j|
&\leq& |f(\alpha_i, \beta_j) - u_i^T v_j| + |u_i^T v_j - x_i^Ty_j| \\
&\leq& \nicefrac \epsilon 2 \|f\| + \epsilon_\text{JL} (C_u + C_v + 1 + \nicefrac{\epsilon}{2})\|f\|.
\eeas
Thus, if we select $\epsilon_\text{JL}$ to be
\[
\epsilon_\text{JL} =
\frac {\nicefrac \epsilon 2} {C_u + C_v + 1 + \nicefrac{\epsilon}{2}},
\]
then we have $|f(\alpha_i, \beta_j) - x_i^T y_j|\leq\epsilon \|f\|$, as desired.

Therefore, the $\epsilon \|f\|$-rank of $X$ is at most the rank of the matrix
$\tilde{X}_{ij} = x_i^T y_j$, which is of rank at most $r$. Here, $r$ is the integer
given by
\[
r = \Bigg\lceil 8 \log(m+n+1) \left(1 + \frac{2(C_u + C_v + 1)}{\epsilon} \right)^2\Bigg\rceil.
\]
\end{proof}

% \subsection{Discussion.} Theorem~\ref{thm-main-detailed} is a first answer to the question:
% ``why are low rank matrices so effective for modeling problems in data analysis and machine learning?"
% Theorem~\ref{thm-main-detailed} can be paraphrased as follows: if the columns and rows of a data table are drawn from independent and identically distributed random variables from some (nice) underlying distributions, then for large enough numbers of samples, the resulting matrix is of low $\epsilon$-rank.

%The rank of the resulting matrix depends on the function $f$ and distributions
%$\mathcal A$ and $\mathcal B$ parametrizing the LVM.
%It also depends on the desired tolerance $\epsilon > 0$ of the low rank approximation.
%The constants $C_u$ and $C_v$ that appear in the statement of Theorem~\ref{thm-main-detailed} tend not to be large for LVMs of interest in machine learning. Crucially, the constants $C_u$ and $C_v$ do not dependent on the dimension $N$ of the latent variables $\alpha$ and $\beta$ and the bound
%grows logarithmically in $m+n$. Therefore, Theorem~\ref{thm-main-detailed}
%shows that nice LVMs are of log-rank.

\begin{remark}
Note that Theorem~\ref{thm-main-detailed} is only interesting when
\[
\min(m,n) > \Bigg\lceil 8 \log(m+n+1) \left(1 + \frac{2(C_u + C_v + 1)}{\epsilon} \right)^2\Bigg\rceil,
\]
since the rank of a matrix is always bounded by its smallest dimension.
Hence, we see Theorem~\ref{thm-main-detailed} is interesting for sufficiently large matrices.
\end{remark}

\subsection{Piecewise nice latent variable models}
The requirement that the function $f$ associated to the LVM be analytic can be relaxed to piecewise analytic.  We call such models piecewise nice LVMs.
\begin{definition} The family of matrices $\mathcal X_{f,\mathcal A,\mathcal B}$
is call a piecewise nice LVM if there exists a finite partition of the distributions
\[
\mathcal A \times \mathcal B = \cup_{\ell=1}^P (\mathcal A_\ell \times \mathcal B_\ell), \quad
(\mathcal A_\ell \times \mathcal B_\ell) \cap (\mathcal A_{\ell'} \times \mathcal B_{\ell'}) = \emptyset,\quad\ell \ne \ell'
\]
so that
\[
f(\alpha,\beta) = f_\ell(\alpha,\beta), \qquad (\alpha,\beta) \in \mathcal A_\ell \times \mathcal B_\ell
\]
with $\mathcal X_{f_\ell,\mathcal A_\ell,\mathcal B_\ell}$ being nice LVMs for $1\leq \ell \leq P$.
\end{definition}

We find that any piecewise nice LVM is also of log-rank.
\begin{theorem}
\label{cor:piecewise}
Let $\mathcal{X}_{f, \mathcal A, \mathcal B}$ be a piecewise nice latent variable model with distributions of $\mathcal{A}$ and $\mathcal{B}$ of bounded support.
Then, for each $0<\epsilon<1$ and for any $X^{(m\times n)} \in \mathcal X_{f, \mathcal A, \mathcal B}$ the $\epsilon \|f\|$-rank of $X^{(m\times n)}$ is no more than
\[
r = \Bigg\lceil 8 \log(m+n+1) \left(1 + \frac{2(C_u + C_v + 1)}{\epsilon} \right)^2\Bigg\rceil,
\]
where $C_u$ and $C_v$ are constants that depend on properties of the latent variable model $\mathcal{X}_{f, \mathcal A, \mathcal B}$.
\end{theorem}

The proof of this theorem is an easy modification of the proof of Theorem~\ref{thm-main-detailed} because the dimension of the projected vectors in the Johnson--Lindenstrauss Lemma is independent of the dimension of the original vectors. For example, we can take
\[
u_i = (0, \ldots, 0, \overbrace{u_i^{(\ell)}}^{\alpha_i \in \mathcal A_\ell}, 0, \ldots, 0),\qquad
v_j = (0, \ldots, 0, \overbrace{v_j^{(\ell)}}^{\beta_j \in \mathcal B_\ell}, 0, \ldots, 0),
\]
where $1\leq i\leq m$ and $1\leq j\leq n$.  Note it is possible that $\alpha_i \in \mathcal A_\ell$ (resp. $\alpha_i \in \mathcal B_\ell$) for multiple $\ell$s, so $u_i$ (resp. $v_i$) may have
more than one nonzero block. We can also take
\[
\hat X_{ij} = u_i^T v_j
= \sum_{l: (\alpha_i, \beta_j) \in \mathcal A_\ell \times \mathcal B_\ell} \left(u_i^{(\ell)}\right)^T v_j^{(\ell)}
= \left(u_i^{(\ell_{ij})}\right)^T v_j^{(\ell_{ij})}
\]
where $\ell_{ij}$ is the unique $\ell$ so that $(\alpha_i, \beta_j) \in \mathcal A_\ell \times \mathcal B_\ell$.
(It is unique because $\{\mathcal A_\ell \times \mathcal B_\ell\}_{\ell=1}^P$ partitions $\mathcal A \times \mathcal B$.)  Lastly, the norms of $u_i$ and $v_j$ are just the sum of the norms of $u^{(\ell)}_i$ and $v^{(\ell)}_j$ so the constants $C_u$ and $C_v$ in the proof are replaced by
$\max_\alpha \sum_{\ell:~ \alpha \in \mathcal A_\ell} C_u^{(\ell)}$ and
$\max_\beta \sum_{\ell:~ \beta \in \mathcal B_\ell} C_v^{(\ell)}$.
% \[
% C_u \leftarrow \max_\alpha \sum_{\ell:~ \alpha \in \mathcal A_\ell} C_u^{(\ell)},\qquad
% C_v \leftarrow \max_\beta \sum_{\ell:~ \beta \in \mathcal B_\ell} C_v^{(\ell)}.
% \]
% $\max_\ell C_u^{(\ell)}$ and $\max_\ell C_v^{(\ell)}$.

\subsection{Symmetric latent variable models}
Above, we noticed a connection between latent variable models and exchangeable
families of matrices.
To understand the rank of \emph{symmetric} exchangeable families of matrices (e.g., graphons),
and the rank of \emph{symmetric} matrices,
we define a symmetric notion of latent variable models:
\begin{definition} A family of matrices $\mathcal X_{f,\mathcal A}$
is a symmetric latent variable model (depending on $f$ and $\mathcal A$)
if for every $X^{(n\times n)}\in\mathcal{X}_{f,\mathcal{A}}$,
\[
(X^{(n \times n)})_{ij}=f(\alpha_i, \alpha_j), \qquad 1\leq i,j\leq n.
\]
\end{definition}
If $\mathcal A$ is compact and $|D^\mu f(\alpha,\alpha')| \leq C M^{|\mu|}\|f\|$,
we say the symmetric LVM is nice. If $\mathcal A = [0,1]$ and $f: [0,1]\times[0,1]\to[0,1]$,
then $\mathcal X_{f,\mathcal A}$ is a graphon \cite{lovasz2006limits}.
Graphons are often used to model processes that generate random graphs,
by interpreting the entries of $X^{(n\times n)} \in \mathcal X_{f,\mathcal{A}}$ as the probability
that a graph on $n$ nodes has an edge between node $i$ and node $j$.

We show any symmetric LVM is of log-rank.
\begin{theorem}\label{cor:graphons}
Let $\mathcal{X}_{f, \mathcal A}$ be a nice symmetric latent variable model and let $0<\epsilon <1$. Then, for $X^{(n\times n)} \in \mathcal X_{f,\mathcal{A}}$,
the $\epsilon \|f\|$-rank of $X^{(n\times n)}$ is no more than
\[
r = \Bigg\lceil 8 \log(2n+1) \left(1 + \frac{2(C_u + C_v + 1)}{\epsilon} \right)^2\Bigg\rceil,
\]
where $C_u$ and $C_v$ are constants which depend on
the latent variable model $\mathcal{X}_{f, \mathcal A}$.
\end{theorem}

The proof of this theorem is nearly identical to the proof of Theorem~\ref{thm-main-detailed},
since we never use independence of $\alpha_i$ and $\beta_j$.

\section{Numerical experiments}
Our theory shows that a matrix generated from a nice LVM is often well-approximated by a matrix of low rank, even if the true latent structure is high dimensional or nonlinear.
%Is there any way to distinguish between low rank arising from ``true'' low dimensional linear structure, and low rank arising from high dimensional nonlinear structure as a mathematical consequence of concentration of measure (as in our theory)?
%
%One heuristic is to measure the growth of the rank of the matrices. For example, subsample $k$ rows and columns of the matrix and compute the rank of the submatrix, then subsample $2k$ rows and columns and compute the rank, and so on. If the rank approaches a constant as the dimension grows, that this gives confidence that the matrix enjoys true low dimensional structure.  On the other hand, if the rank increases with the matrix size, one might suspect the low rank is a consequence of high dimensional or nonlinear latent structure.
However, computing the $\epsilon$-rank for $0<\epsilon<1$ is probably NP-hard~\cite{gillis2017low}, where
\[
{\rm rank}_\epsilon(X) = \min\left\{{\rm rank}(A) : A\in\mathbb{R}^{m\times n}, \text{ } \left\|X -  A\right\|_{\max} \leq \epsilon\right\}.
\]
This makes numerical experiments difficult as our theory is only meaningful for large matrices.

A simple approach to crudely compute ${\rm rank}_\epsilon(X)$ is to approximate $X$ by its truncated SVD, using whatever truncation is necessary so that $\left\|X -  A\right\|_{\max} \leq \epsilon$.
% Other, more sophisticated, heuristic algorithms are proposed in \cite{gillis2017low}.
More formally, define $[X]_r = \argmin_{{\rm rank}(Y) \leq r} \|X - Y\|_2$ and define $\mu_r(X)$ as
\[
\mu_r(X) = \left\|X - [X]_r\right\|_{\max}.
\]
An upper bound on ${\rm rank}_\epsilon(X)$ can be found by selecting the small integer $r$ so that $\mu_r(X) \leq \epsilon$.

This paper provides three different bounds on the $\epsilon$-rank for a matrix
$X \in \reals^{n \times n}$ drawn from a nice LVM with latent factors of dimension $N$.
Lemma~\ref{lem-bounded-rank} shows that ${\rm rank}_\epsilon(X) = \mathcal O(N^N \log(1/\epsilon))$.
Our main result, Theorem~\ref{thm-main-detailed},
shows that ${\rm rank}_\epsilon(X) = \mathcal O(\log n / \epsilon^2)$.
And, of course, we have the trivial bound of ${\rm rank}_\epsilon(X) \leq n$.
Based on these bounds, we should expect that when $N$ is large, then for
sufficiently large $n$, ${\rm rank}_\epsilon(X)$ grows like $\log n$.
On the other hand, for small $n$ or $\epsilon$,
we can have $\log n / \epsilon^2 \gtrsim n$,
and hence we may see that ${\rm rank}_\epsilon(X)$ grows linearly with $n$.

Figure~\ref{fig-large-N} shows both of these behaviors. We realize a matrix by drawing
from a nice LVM with $N=1000$: each latent variable is generated as a random point
on the $N$-dimensional unit sphere,
and we use the function $f(\alpha, \beta) = \exp(-\|\alpha - \beta\|^2)$
to generate matrix entries.
We plot our crude upper bound on ${\rm rank}_\epsilon(X)$ using the values of $\mu_r(X)$ by generating matrices for a range of
tolerances $\epsilon$ and dimensions $n$. For each value of $\epsilon$ and $n$, we randomly draw five matrices and plot the maximum obtained upper bound.
We can see that for small $n$ or $\epsilon$, our upper bound on ${\rm rank}_\epsilon(X)$ grows linearly in the dimension $n$.
On the other hand, we can see that for large $n$ and $\epsilon$, the growth of ${\rm rank}_\epsilon(X)$ is approximately logarithmic in $n$.
%The continuing growth of ${\rm rank}_\epsilon(X)$ with the dimension $n$ is an indication that the
%low rank structure that is observed are generated by draws from a (possibly high dimensional) LVM
%rather than from a ``truly low rank'' model in which low dimensional features interact linearly.

\begin{figure}
\centering
\includegraphics[width=.65\textwidth]{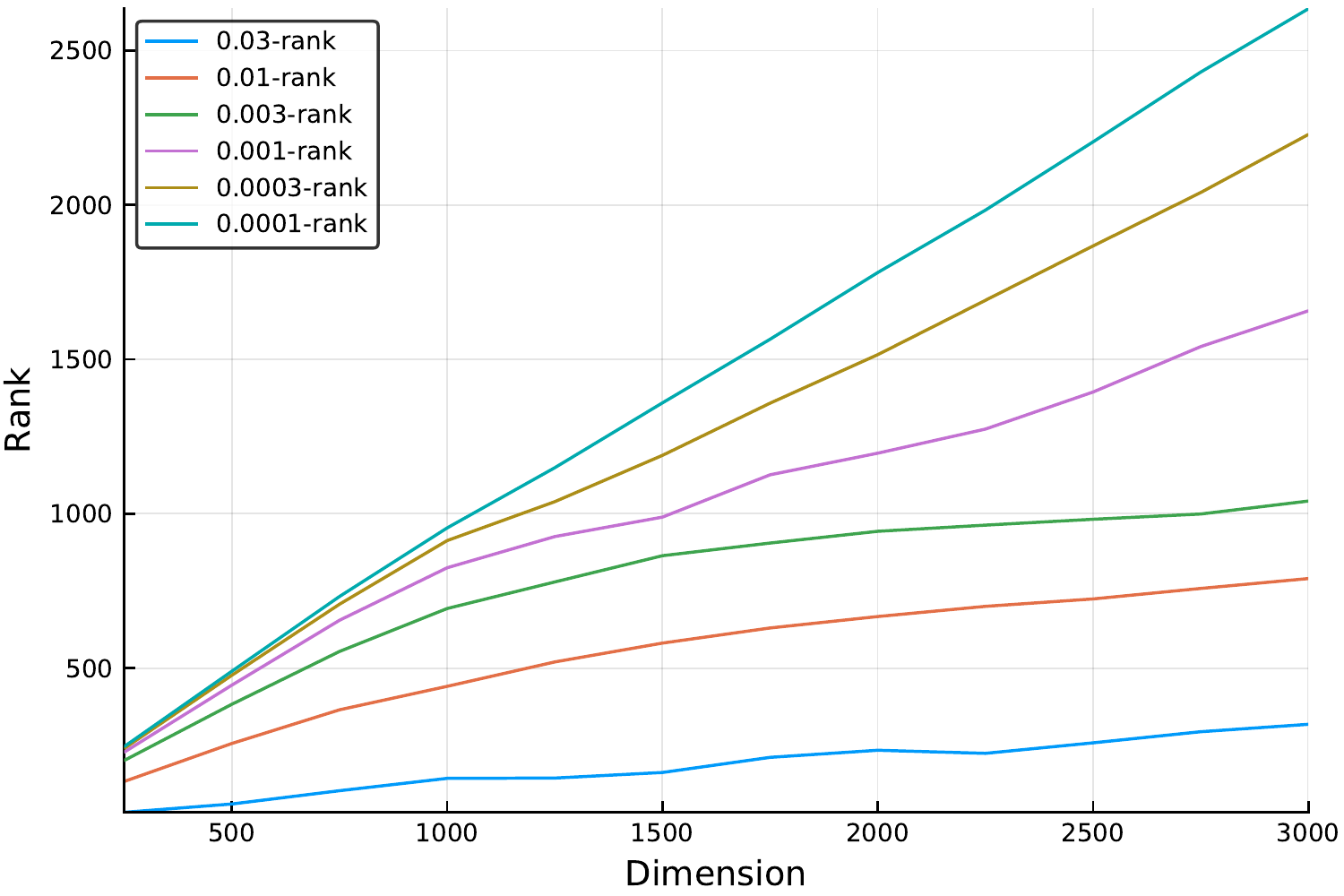}
\caption{\label{fig-large-N}
An upper bound on ${\rm rank}_\epsilon(X)$ for $0.0001\leq \epsilon \leq 0.03$ and $300\leq n\leq 3000$.}
\end{figure}
% <Figure is first one from code/Log rank-high-dim.ipynb>

\section*{Conclusion}
This paper seeks to answer the question:
``Why are low rank techniques so effective for solving problems in data analysis and machine learning?''
Theorem~\ref{thm-main-detailed} provides a partial explanation for its effectiveness:
when rows and columns of the data are drawn from
a nice and consistent distribution, the rank of the resulting matrix cannot increase
very quickly. Formally, we have shown that nice latent variable models give rise to matrices that have
an $\epsilon$-rank that grows only logarithmically with the matrix dimensions, with respect to the maximum absolute entry norm.
This suggests that low rank structure in large datasets is a universal feature and
provides a broad motivation for low rank techniques in data science and machine learning.

%Many questions remain open. How does the rank of the best approximation change
%in the presence of noisy or partial observations from these models?
%What is the sample complexity required for consistent estimation of these models?
%Are higher order tensor latent variable models also of low (tensor) rank, and
%what does that rank depend on? Are the technical assumptions we use here tight, or can they be relaxed?

%We expect these questions to spur exciting developments in both theoretical and applied research.

\section*{Acknowledgements}
We would like to thank Siddhartha Banerjee, Lijun Ding, and Joel Tropp for useful discussions.

\bibliographystyle{siam}
\bibliography{exlog}

\end{document}